%% file: anonymous-submission-latex-2026.tex
\documentclass[letterpaper]{article} 
\usepackage{aaai2026}  
\usepackage{times}  
\usepackage{helvet}  
\usepackage{courier}  
\usepackage[hyphens]{url}  
\usepackage{graphicx} 
\usepackage{natbib}  
\usepackage{caption} 
\usepackage{algorithm}
\usepackage{algorithmic}

\usepackage{newfloat}
\usepackage{listings}

\DeclareCaptionStyle{ruled}{
labelfont=normalfont,
labelsep=colon,
strut=off
} 

\floatstyle{ruled}
\newfloat{listing}{tb}{lst}{}
\floatname{listing}{Listing}

\usepackage{xcolor}
\usepackage{makecell}
\usepackage{pifont}
\usepackage{tcolorbox}
\usepackage{booktabs}
\usepackage{xspace}
\usepackage{outlines}
\usepackage{subfigure}
\usepackage{amsmath}
\usepackage{amssymb}
\usepackage{multirow}
\usepackage{enumitem}
\usepackage{amsthm}

\newtheorem{theorem}{Theorem}

\title{Unsupervised Graph Representation Learning with Complementary View Alignment}

\author{
Zengyi Wo\textsuperscript{\rm 1},
Shiyu Zhang\textsuperscript{\rm 1},
Qiyao Peng\textsuperscript{\rm 2},
Tianpeng Li\textsuperscript{\rm 2},
Xuan Guo\textsuperscript{\rm 2}
}

\affiliations{
\textsuperscript{\rm 1}Nanjing University, Nanjing, China\
\textsuperscript{\rm 2}Tianjin University, Tianjin, China\
}

\usepackage{bibentry}

\begin{document}

\maketitle

\begin{abstract}
Unsupervised graph representation learning aims to derive meaningful node embeddings by capturing both structural and attribute information without relying on labeled data. Existing methods, such as GAEs, have demonstrated effectiveness but typically rely on message-passing mechanisms that assume homophily, leading to performance degradation on heterophilous graphs, where connected nodes exhibit dissimilar features. This homophily bias results in the loss of critical high-frequency components that are essential for identifying heterophilous patterns. To address these challenges, we propose \textsc{AlignGAE}, a novel extension of \textit{MaskGAE} that preserves the full frequency spectrum through complementary view alignment. Our framework introduces a dual-encoder architecture that separately processes structural and attribute information, incorporates node positional encoding to approximate Neighborhood Identity Distribution (NID), and employs dual reconstruction tasks for both edges and node attributes. We further propose theoretically grounded NID alignment strategies that ensure semantic consistency across views while preserving their distinct characteristics. Through comprehensive spectral analysis, we demonstrate that \textsc{AlignGAE} achieves optimal representation properties when the alignment loss converges. Extensive experiments across 12 benchmark datasets validate our approach, showing that \textsc{AlignGAE} outperforms state-of-the-art methods by up to 18.7\% on heterophilous graphs in node classification, while maintaining competitive performance on homophilous graphs. Our results establish a new paradigm for frequency-aware graph representation learning.
\end{abstract}

\input{intro}

\input{related}
\input{Preliminary}
\input{methods}
\input{Experiments}
\input{conclusion}
\input{appendix}

\bibliography{aaai2026}

\input{ReproducibilityChecklist}

\end{document}

%% file: intro.tex
\section{Introduction}
\label{sec:introduction}

Graph-structured data is ubiquitous in many critical domains, such as social networks~\cite{kavanaugh2005community}, transportation systems~\cite{cui2019traffic}, and financial applications~\cite{cheng2022financial}. However, the inherent complexity of graph data, characterized by high-dimensionality and intricate relationships between entities, presents significant challenges for machine learning. Unsupervised Graph Representation Learning (UGRL) has emerged as a promising approach to address these challenges~\cite{hamilton2017representation}. UGRL methods learn low-dimensional node embeddings that preserve both structural and attribute information, without the need for labeled data. These embeddings enable a broad range of downstream tasks, making UGRL highly valuable in real-world applications.

Despite the progress in UGRL, capturing both homophilous and heterophilous patterns remains a difficult problem. In particular, heterophilous graphs, where nodes connected by edges have dissimilar attributes, are challenging to represent effectively. Traditional UGRL techniques often struggle in these scenarios, as they typically assume homophily—nodes that are connected tend to share similar attributes. This assumption leads to performance degradation when applied to heterophilous graphs, where connected nodes exhibit substantial differences.

Recent advancements, such as masked graph autoencoders, have utilized message-passing mechanisms to encode structural patterns through neighborhood aggregation~\cite{kipf2016semi, balcilar2021breaking, velivckovic2017graph}. These methods are based on graph signal processing, where graph signals are decomposed into low-frequency components that capture homophilous patterns and high-frequency components that are critical for heterophilous relationships. However, message-passing operations in these methods act as low-pass filters, emphasizing low-frequency components and suppressing high-frequency components~\cite{bo2021lowfrequencyinformationgraphconvolutional,zhu2020homophilygraphneuralnetworks}. This limitation compromises the effectiveness of these techniques in handling heterophilous graphs, where high-frequency components play a crucial role in distinguishing nodes with different attributes~\cite{morris2021weisfeilerlemanneuralhigherorder, pan2023homophilyreconstructingstructuregraphagnostic, lin2023multiviewgraphrepresentationlearning}.

To overcome these challenges, we introduce \textsc{AlignGAE}, a novel framework that integrates dual-view spectral alignment, allowing it to preserve the full frequency spectrum. Our approach utilizes two complementary views: a neighborhood view that captures structural patterns through message-passing and a node view that preserves attribute information via non-convolutional processing. These views are aligned using a NID matching technique, ensuring semantic consistency while maintaining the distinct characteristics of each view.

Theoretically, \textsc{AlignGAE} maximizes mutual information between the views, ensuring that they share consistent distributions. It effectively captures both low- and high-frequency signals, thanks to its dual-encoder architecture and dual reconstruction tasks. Extensive experiments on a variety of benchmark datasets demonstrate that \textsc{AlignGAE} outperforms state-of-the-art methods, particularly in the context of heterophilous graphs, while maintaining competitive performance on homophilous graphs.

In summary, our contributions are threefold: 
\begin{itemize}
    \item We establish a theoretical connection between spectral decomposition and representation quality, with NID serving as a crucial bridge.
    \item We introduce \textsc{AlignGAE}, a framework that extends masked graph autoencoders by incorporating dual-view spectral alignment to capture both homophilous and heterophilous patterns.
    \item We provide comprehensive experimental validation, showing that \textsc{AlignGAE} excels across a wide range of graph structures, especially in challenging heterophilous settings. This work pushes the boundaries of UGRL by overcoming the limitations of traditional message-passing methods through principled view alignment.
\end{itemize}

%% file: related.tex
\section{Related Work}
\label{sec:related_work}

This section reviews the key methodologies in unsupervised graph representation learning (UGRL), focusing on approaches that inform our spectral decomposition framework. We discuss graph autoencoders and graph contrastive learning methods, emphasizing their limitations and how \textsc{AlignGAE} builds upon these foundations by addressing critical gaps.

\subsection{Graph Autoencoders}
Graph autoencoders (GAEs) have been a central approach for unsupervised graph representation learning. GAEs work by encoding graphs into low-dimensional embeddings and reconstructing structural or attribute information, with early work like GAE and VGAE~\cite{kipf2016variational} setting the foundation through link prediction objectives. Later methods, such as GALA~\cite{park2019symmetric}, expanded this to feature reconstruction. However, while these methods made significant strides in preserving local structure, they often prioritize proximity-based relationships and fail to capture the global structural properties of graphs effectively~\cite{veličković2018deepgraphinfomax}. This limitation becomes particularly evident when dealing with heterophilous graphs, where connected nodes differ significantly in attributes, making simple proximity-based encoding insufficient. In an attempt to address these issues, recent approaches like MaskGAE~\cite{li2023whatsmaskunderstandingmasked} and GraphMAE~\cite{hou2022graphmae} introduced masked autoencoder frameworks, which utilize edge perturbation and structured masking to maintain topological integrity during self-supervision. However, these approaches tend to focus on low-frequency signals, often overlooking the high-frequency components crucial for heterophilous graph representation, as discussed in Section~\ref{subsec:spectral}. This gap limits their effectiveness on more complex graph structures, which is where \textsc{AlignGAE} provides a novel solution.

\subsection{Graph Contrastive Methods}
Graph contrastive learning (GCL) has emerged as a dominant paradigm for self-supervised learning of graph representations. GCL methods, like DGI~\cite{veličković2018deepgraphinfomax}, maximize mutual information between different views of the graph, using corrupted graph versions to train the model. Techniques like MVGRL~\cite{hassani2020mvgrl} and GRACE~\cite{GRACE} introduced various augmentations to improve robustness, such as edge dropping and feature masking. However, while these methods have achieved empirical success, they rely heavily on manually crafted augmentations whose effectiveness can vary significantly across different types of graphs, particularly in heterophilous graphs. The primary issue is that these methods, much like traditional message-passing approaches, are often limited by their spectral filtering properties, which disregard high-frequency components critical for representing heterophilous relationships~\cite{liu2023beyond}. 

\textsc{AlignGAE} extends the masked autoencoder framework by incorporating a dual-view architecture that preserves both low- and high-frequency signals through NID alignment. This enables the model to capture both homophilous and heterophilous patterns effectively without requiring task-specific adaptations. By overcoming the limitations of existing methods, \textsc{AlignGAE} provides a principled way to handle complex graph structures across both homophilous and heterophilous domains, significantly improving performance on challenging graph tasks.

%% file: Preliminary.tex
\section{Preliminary}
\label{sec:preliminary}

This section outlines the theoretical foundations underlying \textsc{AlignGAE}, focusing on key concepts such as the NID, spectral decomposition, and optimal representation conditions. These concepts are central to the design of our method, as detailed in Section Methodology. We provide a rigorous framework that exploits both homophilous and heterophilous signals through spectral decomposition, contrasting prior works that treat heterophilous graphs as problematic. The detailed proofs are included in the Appendix.

\noindent\textbf{Notation} We define the graph $\mathcal{G} = (\mathcal{V}, \mathcal{E})$, where $\mathcal{V} = \{v_1, \dots, v_n\}$ represents the set of nodes and $\mathcal{E} \subseteq \mathcal{V} \times \mathcal{V}$ denotes the set of edges. The number of nodes and edges is given by $|\mathcal{V}| = n$ and $|\mathcal{E}| = m$, respectively. The adjacency matrix $\mathbf{A} \in \{0,1\}^{n \times n}$ is defined such that $A_{ij} = 1$ if $(v_i, v_j) \in \mathcal{E}$, and $A_{ij} = 0$ otherwise. The degree matrix $\mathbf{D} \in \mathbb{R}^{n \times n}$ is diagonal, with $D_{ii} = \sum_j A_{ij}$. Let $\mathbf{X} \in \mathbb{R}^{n \times d}$ represent the node attribute matrix, where $\mathbf{x}_i \in \mathbb{R}^d$ is the feature vector of node $v_i$.

\noindent\textbf{Graph Signal Spectral Decomposition}
\label{spectral}
Consider a graph signal $\mathbf{x} \in \mathbb{R}^n$ defined on the nodes of $\mathcal{G}$. The normalized graph Laplacian is given by
\begin{equation}
\mathbf{L} = \mathbf{I} - \mathbf{D}^{-1/2} \mathbf{A} \mathbf{D}^{-1/2} = \mathbf{U} \mathbf{\Lambda} \mathbf{U}^\top,
\end{equation}
where $\mathbf{U} \in \mathbb{R}^{n \times n}$ contains the orthonormal eigenvectors $\{\mathbf{u}_k\}_{k=1}^n$, and $\mathbf{\Lambda} = \text{diag}(\lambda_1, \dots, \lambda_n)$ is the diagonal matrix of eigenvalues, where $0 = \lambda_1 \leq \lambda_2 \leq \dots \leq \lambda_n \leq 2$.

Any graph signal $\mathbf{x}$ has a unique spectral decomposition:
\begin{equation}
\mathbf{x} = \sum_{k=1}^n \hat{x}_k \mathbf{u}_k, \quad \text{where} \quad \hat{x}_k = \langle \mathbf{x}, \mathbf{u}_k \rangle.
\end{equation}
We divide the frequency spectrum at the threshold $\lambda^* = 1$, the midpoint of the eigenvalue range $[0, 2]$:
\begin{equation}
\begin{split}
\mathcal{F}_L &= \{k : \lambda_k \leq \lambda^*\}, \quad \text{(low-frequency components)}, \\
\mathcal{F}_H &= \{k : \lambda_k > \lambda^*\}, \quad \text{(high-frequency components)}.
\end{split}
\end{equation}
As established in~\cite{liu2023beyond}, homophilous signals are primarily captured in $\mathcal{F}_L$, while heterophilous signals are concentrated in $\mathcal{F}_H$. This spectral separation forms the basis for our approach.

Traditional message-passing graph neural networks (MPGNNs) act as low-pass filters, suppressing high-frequency components. This is represented in the spectral domain as:
\begin{equation}
\mathbf{H}^{(l)} = \mathbf{U} g(\mathbf{\Lambda}) \mathbf{U}^\top \mathbf{X},
\end{equation}
where $g(\lambda_k) \approx 0$ for $\lambda_k > \lambda^*$, effectively discarding high-frequency signals. Consequently, for heterophilous graphs, the reconstruction error due to spectral truncation is non-zero:
\begin{equation}
\|\mathbf{X} - \mathbf{U}_{\mathcal{F}_L} \mathbf{U}_{\mathcal{F}_L}^\top \mathbf{X}\|_F^2 > 0.
\end{equation}
This spectral incompleteness leads to the performance degradation of masked graph autoencoders on heterophilous graphs. To address this, \textsc{AlignGAE} employs a dual-view spectral decomposition that explicitly preserves both low- and high-frequency signals, as shown in Theorem~\ref{thm:optimal}.

\noindent\textbf{Neighborhood Identity Distribution} The NID for a node $v_i$ is the probability distribution over pairwise feature distances in its neighborhood:
\begin{equation}
p_i^g(\delta) = \frac{1}{|\mathcal{N}(i)|} \sum_{j \in \mathcal{N}(i)} \delta_D\left( \|\mathbf{x}_i - \mathbf{x}_j\|_2 - \delta \right),
\end{equation}
where $\delta_D(\cdot)$ denotes the Dirac delta function. Homophilous nodes exhibit a sharp NID peak near $\delta = 0$, while heterophilous nodes have broader, multimodal distributions.

\begin{theorem}
\label{thm:nid}
Let $\mathbf{P} \in \mathbb{R}^{n \times k}$ be the positional encoding derived from $r$ random walks per node. The empirical NID $\hat{p}_i(\delta)$ computed from random-walk trajectories approximates the true NID $p_i^g(\delta)$ with bounded total variation error:
\begin{equation}
\mathbb{E}_{\mathcal{W}_i} \left[ \| p_i^g(\delta) - \hat{p}_i(\delta) \|_{\mathrm{TV}} \right] \leq \mathcal{O}\left( \frac{1}{\sqrt{r}} \right).
\end{equation}
\end{theorem}

Theorem~\ref{thm:nid} justifies the use of random-walk positional encodings, showing that they capture feature inconsistency patterns, essential for aligning homophilous and heterophilous representations.

\noindent\textbf{Optimal Representation Conditions:} The following theorem formalizes the conditions under which \textsc{AlignGAE} achieves optimal representation learning by maintaining both spectral fidelity and structural consistency.

\begin{theorem}
\label{thm:optimal}
Suppose the alignment loss $\mathcal{L}_{\text{align}} \to 0$ in \textsc{AlignGAE}. Then, the learned representations $\mathbf{Z}_g$ (geometric view) and $\mathbf{Z}_m$ (message-passing view) satisfy maximum mutual information $\mathrm{MI}(\mathbf{Z}_g; \mathbf{Z}_m) = \mathrm{H}(\mathbf{Z}_g) = \mathrm{H}(\mathbf{Z}_m)$, identical NID distributions $p_i^g(\delta) = p_i^m(\delta)$ for all $i \in \mathcal{V}$, and full frequency spectrum capture $\|\mathbf{X} - \mathbf{U} \mathbf{U}^\top (\mathbf{Z}_g + \mathbf{Z}_m)\|_F^2 = 0$.
\end{theorem}

Theorem~\ref{thm:optimal} provides the theoretical basis for our dual-encoder architecture, dual reconstruction tasks, and NID-based alignment. It ensures that \textsc{AlignGAE} preserves both homophilous (low-frequency) and heterophilous (high-frequency) signals, addressing the spectral limitations of standard masked graph autoencoders.

%% file: methods.tex
\section{Methodology}
\label{sec:methodology}
\begin{figure*}[htbp]
    \centering
    \includegraphics[width=0.98\textwidth, height=0.3\textheight, keepaspectratio]{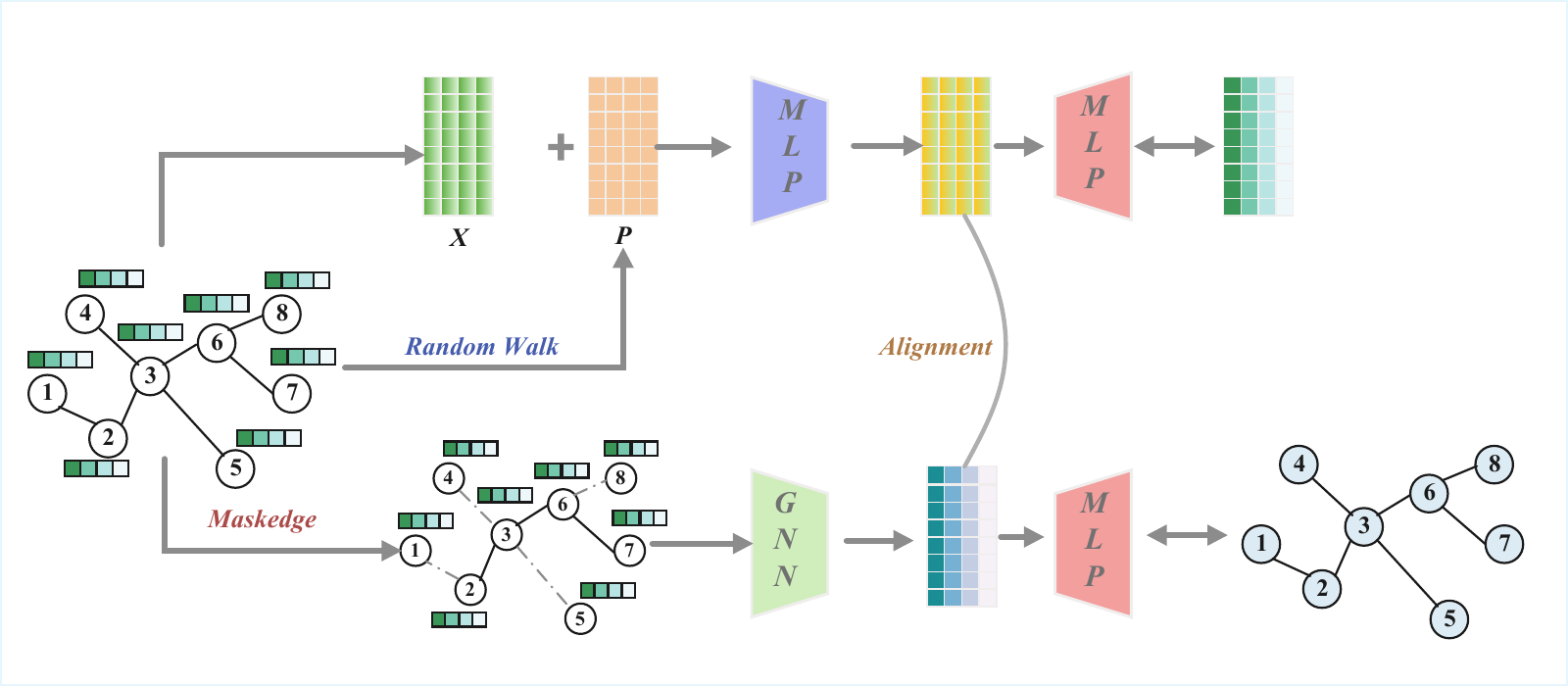}
    \caption{Overview of the \textsc{AlignGAE} framework.}
    \label{fig:frame}
\end{figure*}
\subsection{Overview}
\label{subsec:overview}

Masked graph autoencoders have excelled in graph representation learning by capturing local structural patterns through message-passing. However, as discussed in Section Preliminary, these models predominantly focus on low-frequency components ($\mathcal{F}_L$), neglecting high-frequency signals ($\mathcal{F}_H$) essential for heterophilous relationships. To address this, \textsc{AlignGAE} extends the masked graph autoencoder framework with dual-view spectral decomposition and alignment, enabling the model to capture both homophilous and heterophilous patterns. It processes graph data through two complementary encoders: one for structural patterns (neighborhood view) and another for node attributes (node view), aligning them via NID matching to ensure optimal representation. This approach enhances the original framework, preserving its strength in structural learning while improving its ability to handle diverse graph structures. The dual reconstruction tasks ensure spectral completeness, preserving both low- and high-frequency components, while the NID alignment guarantees semantic consistency between the views, resulting in more robust and generalizable node embeddings.

\subsection{Node Positional Embedding}
\label{subsec:positional_embedding}

To integrate structural and attribute information, \textsc{AlignGAE} introduces node positional embeddings that capture global topological context, complementing local neighborhood structures. These embeddings provide a grounded approximation of the NID, which quantifies feature inconsistencies within a node's neighborhood, as formalized in Theorem~\ref{thm:nid}.

Positional embeddings are derived from random walks, simulating diffusion processes over the graph. For each node $v_i \in \mathcal{V}$, we perform $r$ random walks of length $l$ starting from $v_i$. The random walk sequences are used to train a Skip-gram model with negative sampling:
\begin{equation}
    p(v_j | v_i) = \frac{\exp(\mathbf{z}_j^\top \mathbf{z}_i)}{\sum_{k \in \mathcal{V}} \exp(\mathbf{z}_k^\top \mathbf{z}_i)},
    \label{eq:skipgram}
\end{equation}
where $p(v_j | v_i)$ is the probability of node $v_j$ appearing in the context of node $v_i$ within a random walk. The resulting node embeddings $\mathbf{P} = [\mathbf{z}_1, \mathbf{z}_2, \dots, \mathbf{z}_n]^\top \in \mathbb{R}^{n \times d_p}$ provide complementary information to node features, enhancing the input to the node view encoder. 
This method captures global topological roles of nodes, which is particularly beneficial for heterophilous graphs where relationships depend more on global positioning than local homophily.

\subsection{Dual-Encoder Architecture}
\label{subsec:dual_encoder}

\textsc{AlignGAE} employs a dual-encoder architecture that explicitly encodes two complementary views of the graph, implementing the spectral decomposition principle outlined in Section Preliminary. This design ensures that both homophilous (low-frequency) and heterophilous (high-frequency) signals are preserved, addressing the spectral completeness requirement.

\subsubsection{Edge Masking}

To improve structural input and promote learning of long-range dependencies, we adopt a topology random masking strategy. This method extends standard masked graph autoencoder approaches by better preserving heterophilous signals. Root nodes are sampled using a Bernoulli distribution with a mask ratio $p = 0.7$, ensuring sufficient structural retention while maintaining meaningful occlusion. From each root node, random walks of length $l = 3$ select edges, forming a mask matrix $\mathbf{M} \in \{0, 1\}^{n \times n}$, where $M_{ij} = 1$ if edge $(v_i, v_j)$ is retained, and $0$ otherwise. The masked adjacency matrix is computed as:
\begin{equation}
    \tilde{\mathbf{A}} = \mathbf{A} \circ \mathbf{M},
    \label{eq:masked_adj}
\end{equation}
where $\circ$ denotes the Hadamard product. The masked graph $\mathcal{G}_a = (\tilde{\mathbf{A}}, \mathbf{X})$ is input to the neighborhood view encoder, encouraging the model to infer missing connections from visible edges. This improves its ability to capture higher-order structural dependencies while preserving spectral characteristics.

\subsubsection{Node View Encoder}

The node view encoder preserves individual node characteristics, which are crucial for heterophilous graphs, where aggregation may obscure differences in node attributes. The encoder processes node features $\mathbf{X} \in \mathbb{R}^{n \times d}$ and positional embeddings $\mathbf{P} \in \mathbb{R}^{n \times d_p}$ using a shared nonlinear mapper $\psi$:
\begin{equation}
    \mathbf{H} = \psi(\mathbf{X}; \theta_1), \quad \mathbf{S} = \psi(\mathbf{P}; \theta_2),
    \label{eq:mapper}
\end{equation}
where $\psi$ consists of a linear layer followed by ELU activation, and $\theta_1$, $\theta_2$ are learnable parameters. The transformed features are concatenated and processed by an MLP to produce node-wise embeddings:
\begin{equation}
    \mathbf{Z}_m = \text{MLP}(\text{CONCAT}(\mathbf{H}, \mathbf{S}); \omega_1),
    \label{eq:mlp}
\end{equation}
where $\omega_1$ represents the MLP's parameters. This design ensures $\mathbf{Z}_m$ retains high-frequency signals, including attribute variations and positional nuances, without the smoothing effects of graph convolutions, directly addressing the spectral deficiency.

\subsubsection{Neighborhood View Encoder}

The neighborhood view encoder utilizes a two-layer Graph Convolutional Network (GCN) with learnable skip connections to process the masked graph $\mathcal{G}_a$, producing structural embeddings:
\begin{equation}
    \mathbf{Z}_g = \text{MPGNN}(\mathbf{X}, \tilde{\mathbf{A}}; \omega_2),
    \label{eq:mpgnn}
\end{equation}
where $\omega_2$ are learnable parameters. Each GCN layer is defined as:
\begin{equation}
    \mathbf{H}^{(l+1)} = \sigma\left(\alpha_l \mathbf{H}^{(l)} + (1 - \alpha_l) \tilde{\mathbf{D}}^{-\frac{1}{2}} \tilde{\mathbf{A}} \tilde{\mathbf{D}}^{-\frac{1}{2}} \mathbf{H}^{(l)} \mathbf{W}^{(l)}\right),
    \label{eq:mpgnn_layer}
\end{equation}
where $\mathbf{H}^{(l)}$ denote the hidden state at layer $l$, $\tilde{\mathbf{D}}$ the degree matrix of the masked adjacency matrix $\tilde{\mathbf{A}}$, $\mathbf{W}^{(l)}$ the weight matrix, $\sigma$ the ELU activation, and $\alpha_l$ a learnable skip-connection parameter. This formulation acts as a low-pass filter, emphasizing homophilous patterns, with $\tilde{\mathbf{A}}$ capturing structural dependencies. The learnable skip connections enable adaptive frequency filtering, offering a more flexible spectral response than standard MPGNNs. The node-wise embeddings $\mathbf{Z}_m$ and structural embeddings $\mathbf{Z}_g$ jointly form a comprehensive representation: $\mathbf{Z}_m$ retains node-specific details, while $\mathbf{Z}_g$ encodes collective structural insights. This dual-view approach ensures \textsc{AlignGAE} captures the full frequency spectrum, with the neighborhood view targeting low-frequency $\mathcal{F}_L$ signals and the node view preserving high-frequency $\mathcal{F}_H$ components, aligned via NID matching (see Section Preliminary).


\subsection{Dual Reconstruction Tasks}
\label{subsec:reconstruction}

To ensure spectral completeness and robust learning, \textsc{AlignGAE} employs dual reconstruction tasks that effectively capture both graph topology and node attributes, ensuring that both homophilous and heterophilous signals contribute to the learned representations.

\subsubsection{Edge Reconstruction}

To reconstruct the graph topology, \textsc{AlignGAE} predicts masked edges using embeddings from both structural and positional encoders. The edge decoder integrates these embeddings to produce edge predictions:
\begin{equation}
    p(i,j|\hat{\mathcal{G}}) = \sigma\left(\text{MLP}\left( (\mathbf{z}_i^g \circ \mathbf{z}_j^g) \oplus (\mathbf{p}_i \circ \mathbf{p}_j); \phi_1 \right)\right),
    \label{eq:edge}
\end{equation}
where $\mathbf{z}_i^g \in \mathbf{Z}_g$ and $\mathbf{p}_i \in \mathbf{P}$ are the structural and positional embeddings, $\circ$ and $\oplus$ represent element-wise multiplication and concatenation, respectively, and $\sigma$ is the sigmoid function. This cross-view approach leverages low-frequency structural signals and positional context to improve edge prediction.

The reconstruction loss uses binary cross-entropy:
\begin{align}
    \mathcal{L}_{\text{edge}} =\ 
    & -\frac{1}{|\mathcal{E}_m|} \sum_{(i,j) \in \mathcal{E}_m} \log p(i,j|\hat{\mathcal{G}}) \notag \\
    & -\frac{1}{|\mathcal{E}_n|} \sum_{(i,j) \in \mathcal{E}_n} \log \left(1 - p(i,j|\hat{\mathcal{G}})\right),
    \label{eq:edge_loss}
\end{align}

\subsubsection{Node Attribute Reconstruction}

For node attributes, \textsc{AlignGAE} reconstructs the concatenated features $\mathbf{X} \oplus \mathbf{P}$ using node embeddings $\mathbf{Z}_m$. The decoder is defined as:
\begin{equation}
    \mathbf{\hat{X}} = \text{MLP}(\mathbf{Z}_m; \phi_2),
    \label{eq:node_decoder}
\end{equation}
where $\phi_2$ are learnable parameters. We use a scaled cosine error (SCE) loss to account for differences between attributes and positional embeddings:
\begin{equation}
    \mathcal{L}_{\text{node}} = \frac{1}{n} \sum_{i \in \mathcal{V}} \left(1 - \frac{\mathbf{\hat{x}}_i \cdot (\mathbf{x}_i \oplus \mathbf{p}_i)}{\|\mathbf{\hat{x}}_i\| \|\mathbf{x}_i \oplus \mathbf{p}_i\|}\right)^\alpha,
    \label{eq:node}
\end{equation}
with $\alpha = 2$.

\subsection{Neighborhood Identity Distribution Alignment}
\label{subsec:nid_alignment}

Integrating structural and attribute information to generate coherent node embeddings is a central challenge in graph representation learning, especially in heterophilous graphs where neighbors often exhibit dissimilar features. To address this, we introduce \textit{NID alignment}, which ensures that structural and attribute embeddings are aligned while maintaining their complementary characteristics.

\noindent\textbf{Aligning Structural and Attribute Views.} 
In conventional graph neural networks (GNNs), structural and attribute information are processed separately, often resulting in lost interactions between them. In \textsc{AlignGAE}, we unify these views by aligning the NIDs of the structural and attribute embeddings. This ensures consistent contributions from both the graph structure and node attributes, enabling the model to capture both homophilous (low-frequency) and heterophilous (high-frequency) patterns.

\noindent\textbf{Alignment Strategies.} 
We propose three alignment strategies for comparing the structural and attribute NIDs:

\begin{itemize}[noitemsep]
    \item \textbf{KL Divergence.} The Kullback-Leibler (KL) divergence measures the dissimilarity between two probability distributions. For the structural NID \( p_i^g(\delta) \) and attribute NID \( p_i^m(\delta) \), it is defined as:
    \begin{equation}
        D_{\text{KL}}\left( p_i^g(\delta) \parallel p_i^m(\delta) \right) = \int_{\mathcal{R}} p_i^g(\delta) \log \frac{p_i^g(\delta)}{p_i^m(\delta)} \, d\delta
        \label{eq:kl_divergence}
    \end{equation}
    Minimizing this divergence aligns the distributions of structural and attribute embeddings.

    \item \textbf{Wasserstein Distance.} The Wasserstein distance, or Earth Mover's Distance, quantifies the minimal "effort" to transform one distribution into another:
    \begin{align}
    W_1\left( p_i^g(\delta), p_i^m(\delta) \right)
    =\ & \inf_{\gamma \in \Gamma^{\scriptscriptstyle (p_i^g,\, p_i^m)}} 
    \int_{\mathcal{R} \times \mathcal{R}} 
    \|\delta_1 - \delta_2\|_1\, d\gamma(\delta_1, \delta_2)
    \label{eq:wasserstein}
\end{align}

    Minimizing this distance ensures alignment with minimal transformation cost.

    \item \textbf{Contrastive Alignment.} This method maximizes the similarity between the structural and attribute embeddings of the same node while minimizing the similarity between different nodes. The contrastive loss is given by:
    \begin{equation}
        \mathcal{L}_{\text{cont}} = -\frac{1}{n} \sum_{i \in \mathcal{V}} \log \frac{\exp(\text{sim}(\mathbf{z}_i^g, \mathbf{z}_i^m)/\tau)}{\sum_{j \in \mathcal{N}(i) \cup \{i\}} \exp(\text{sim}(\mathbf{z}_i^g, \mathbf{z}_j^m)/\tau)}
        \label{eq:contrastive_loss}
    \end{equation}
    This loss ensures that embeddings of the same node in both views are maximally similar, while embeddings of different nodes are distinguished.
\end{itemize}

The NID alignment in \textsc{AlignGAE}, as formalized in Theorem~\ref{thm:optimal}, aims to preserve both low- and high-frequency signals, ensuring structural and attribute embeddings capture complementary information. By aligning NIDs, \textsc{AlignGAE} integrates structural patterns and attribute details into robust, generalizable node representations. The structural view captures low-frequency components via neighborhood relationships, while the attribute view encodes high-frequency components. This alignment is vital for both homophilous and heterophilous graphs: in heterophilous graphs, it preserves meaningful structural relationships despite dissimilar node attributes, avoiding over-smoothing; in homophilous graphs, it reinforces shared neighbor characteristics. Thus, NID alignment enables \textsc{AlignGAE} to effectively combine both signal types, supporting robust node representation learning across diverse graph structures.

\subsection{Objective Function}
\label{subsec:objective}

The objective function of \textsc{AlignGAE} balances graph topology and node attribute reconstruction with semantic consistency across its dual-encoder architecture. The total loss combines reconstruction and alignment objectives:
\begin{equation}
    \mathcal{L} = \mathcal{L}_{\text{node}}+ \lambda\mathcal{L}_{\text{edge}} + \mu\mathcal{L}_{\text{align}},
    \label{eq:objective}
\end{equation}
where $\mathcal{L}_{\text{align}}$ is the alignment loss. Hyperparameters $\lambda$ and $\mu$ are tuned to achieve optimal trade-offs across different datasets.

This formulation enables \textsc{AlignGAE} to produce robust representations that satisfy the optimal conditions for spectral completeness, ensuring strong performance on both homophilous and heterophilous graphs.

%% file: Experiments.tex
\section{Experiments}
\label{experiments}


\subsection{Datasets}
\label{subsec:datasets}


To evaluate \textsc{AlignGAE}'s effectiveness, we conduct experiments on eleven benchmark datasets spanning four graph domains: citation, co-authorship, departmental webpage, and Wikipedia page networks. These datasets vary in scale and structure, categorized into homophilous graphs, where nodes with similar labels are more likely connected, and heterophilous graphs, where connections often span dissimilar labels. For node classification, we use standard public splits from prior work \cite{kipf2017semisupervisedclassificationgraphconvolutional,pei2020geomgcngeometricgraphconvolutional,yang2016revisitingsemisupervisedlearninggraph} or widely accepted splits \cite{Zhu_2021,thakoor2023largescalerepresentationlearninggraphs} to ensure comparability. Key dataset statistics, including node and edge counts, feature dimensionality, and class numbers, are summarized in Table~\ref{tab:dataset_stats}, grouped by structural homophily.

\begin{table}[t]
  \centering
  \caption{Summary of dataset characteristics.}
  \label{tab:dataset_stats}
  \setlength{\tabcolsep}{3mm} 
  \begin{tabular}{lrrrr}
    \toprule
    \textbf{Dataset} & \textbf{Nodes} & \textbf{Edges} & \textbf{Feat.} & \textbf{Cls.} \\
    \midrule
    \multicolumn{5}{l}{\textit{Homophilic Graphs}} \\
    Cora             & 2,708  & 10,556   & 1,433 & 7  \\
    Citeseer         & 3,327  & 9,104    & 3,703 & 6  \\
    CS               & 18,333 & 163,788  & 6,805 & 15 \\
    Physics          & 34,493 & 495,924  & 8,415 & 5  \\
    \midrule
    \multicolumn{5}{l}{\textit{Heterophilic Graphs}} \\
    Cornell          & 183    & 295      & 1,703 & 5  \\
    Wisconsin        & 251    & 499      & 1,703 & 5  \\
    Texas            & 183    & 309      & 1,703 & 5  \\
    Chameleon        & 2,277  & 36,051   & 2,325 & 5  \\
    Squirrel         & 5,201  & 216,933  & 2,089 & 5  \\
    \bottomrule
  \end{tabular}
\end{table}

\subsection{Baselines}
\label{subsec:baselines}

To evaluate \textsc{AlignGAE}, we compare its performance against a range of unsupervised graph representation learning methods, categorized into three paradigms: topology-only embedding, graph contrastive learning, and graph autoencoders. Topology-based methods, such as \textit{DeepWalk}~\cite{perozzi2014deepwalk}, capture local structural patterns using random walks. In graph contrastive learning, methods like \textit{DGI}~\cite{veličković2018deepgraphinfomax}, \textit{MVGRL}~\cite{hamilton2018inductiverepresentationlearninglarge}, and \textit{GRACE}~\cite{GRACE} learn robust node representations by contrasting different graph views, while approaches such as \textit{GCA}~\cite{zhu2021graph}, \textit{GREET}~\cite{liu2023beyond}, and \textit{HGRL}~\cite{chen2022towards} refine augmentations to mitigate issues like over-smoothing. Methods like \textit{GBT}~\cite{Bielak_2022} and \textit{BGRL}~\cite{BGRL} leverage bootstrap/self-distillation, while \textit{DSSL}~\cite{zhu2021dssl} and \textit{SP-GCL}~\cite{SP-GCL} incorporate structural and spectral priors. In the graph autoencoder category, methods such as \textit{GAE}~\cite{hamilton2018inductiverepresentationlearninglarge} and \textit{MaskGAE}~\cite{MaskGAE} focus on reconstructing graph structures, while \textit{GraphMAE}~\cite{hou2022graphmae} and \textit{GiGaMAE}~\cite{shi2023gigamae} extend these methods to include node attribute reconstruction and joint edge/node masking. \textit{MVGE}~\cite{lin2023MVGE} and \textit{MVMI}~\cite{fan2022mvmi} integrate multiple attribute perspectives for enhanced representation learning.

\subsection{Training and Evaluation Protocol}
\label{subsec:training_evaluation}

To evaluate the quality of unsupervised embeddings for node classification, we employ a two-stage protocol. Models are pretrained in an unsupervised manner on the full graph to generate 256-dimensional node embeddings over 500 epochs, using the Adam optimizer with a learning rate of $10^{-2}$ or $10^{-3}$ and a weight decay of $5\times10^{-5}$. The best checkpoint is selected every 25 epochs based on a held-out validation set. These frozen embeddings are then used as input features for a one-vs-rest logistic regression classifier. We adopt publicly available train/validation/test splits~\cite{thakoor2023largescalerepresentationlearninggraphs, kipf2016semi, hamilton2017representation}, conduct experiments with 10 random seeds, and report mean accuracy with standard deviation. For \textsc{GDAE}, hyperparameters are set to $\lambda_1 = 0.1$ and $\lambda_2 = 10^{-3}$, while baselines use their default hyperparameters.
\subsection{Experimental Results}
\label{subsec:results}

\begin{table}[t]
\centering
\caption{Node classification accuracy (\%) on homophilous datasets.}
\label{tab:exp-results-part1}
\small 
\setlength{\tabcolsep}{2mm} 
\begin{tabular}{lcccc}
\toprule
\textbf{Method} & \textbf{Cora} & \textbf{Cite.} & \textbf{CS} & \textbf{Phys.} \\
\midrule
DeepWalk   & 67.6 $\pm$ 1.5 & 47.9 $\pm$ 1.7 & 83.8 $\pm$ 0.3 & 89.9 $\pm$ 0.3 \\
GRACE      & 80.8 $\pm$ 1.0 & 69.7 $\pm$ 1.4 & 92.2 $\pm$ 0.2 & 93.2 $\pm$ 0.2 \\
DGI        & 78.4 $\pm$ 0.7 & 70.9 $\pm$ 1.5 & 93.5 $\pm$ 0.2 & 93.9 $\pm$ 0.4 \\
MVGRL      & 82.3 $\pm$ 1.2 & 68.2 $\pm$ 1.0 & 92.1 $\pm$ 0.4 & 95.2 $\pm$ 0.1 \\
GREET      & 82.9 $\pm$ 0.7 & 72.6 $\pm$ 0.7 & 94.0 $\pm$ 0.2 & 95.1 $\pm$ 0.2 \\
HGRL       & 78.4 $\pm$ 0.9 & 67.7 $\pm$ 1.4 & 92.0 $\pm$ 0.2 & 93.8 $\pm$ 0.4 \\
GCA        & 80.3 $\pm$ 1.0 & 67.8 $\pm$ 1.6 & 92.9 $\pm$ 0.2 & 95.3 $\pm$ 0.3 \\
MVMI       & 73.8 $\pm$ 1.5 & 63.1 $\pm$ 2.8 & 89.1 $\pm$ 1.0 & 94.4 $\pm$ 0.3 \\
MVGE       & 81.2 $\pm$ 0.9 & 71.5 $\pm$ 0.7 & 92.8 $\pm$ 0.2 & 94.4 $\pm$ 0.3 \\
GAE        & 78.0 $\pm$ 1.5 & 58.5 $\pm$ 1.1 & 82.8 $\pm$ 1.7 & 93.8 $\pm$ 0.3 \\
GraphMAE   & 81.2 $\pm$ 0.6 & 70.8 $\pm$ 0.5 & 90.2 $\pm$ 0.5 & 94.1 $\pm$ 0.1 \\
MaskGAE    & 82.7 $\pm$ 0.9 & 71.5 $\pm$ 0.9 & 92.8 $\pm$ 0.2 & 95.5 $\pm$ 0.2 \\
S2GAE      & 81.1 $\pm$ 0.8 & 71.1 $\pm$ 0.8 & 90.9 $\pm$ 0.3 & 94.7 $\pm$ 0.1 \\
GiGaMAE    & 81.8 $\pm$ 0.8 & 66.2 $\pm$ 1.5 & 91.6 $\pm$ 0.1 & 94.6 $\pm$ 0.2 \\
\midrule
\textsc{AlignGAE} & \textbf{84.7 $\pm$ 0.5} & \textbf{73.2 $\pm$ 0.5} & \textbf{94.3 $\pm$ 0.2} & \textbf{95.9 $\pm$ 0.2} \\
\bottomrule
\end{tabular}
\end{table}

\begin{table}[t]
\centering
\caption{Node classification accuracy (\%) on heterophilous datasets.}
\label{tab:exp-results-part2}
\small 
\setlength{\tabcolsep}{0.5mm} 
\begin{tabular}{lccccc}
\toprule
\textbf{Method} & \textbf{Cornell} & \textbf{Wisc.} & \textbf{Texas} & \textbf{Cham.} & \textbf{Squir.} \\
\midrule
DeepWalk   & 41.1 $\pm$ 6.6 & 48.0 $\pm$ 6.6 & 58.4 $\pm$ 4.1 & 42.0 $\pm$ 2.4 & 28.7 $\pm$ 1.4 \\
GRACE      & 45.4 $\pm$ 4.0 & 54.3 $\pm$ 5.2 & 58.7 $\pm$ 4.8 & 45.0 $\pm$ 2.4 & 30.1 $\pm$ 1.6 \\
DGI        & 58.4 $\pm$ 4.7 & 69.6 $\pm$ 4.7 & 68.9 $\pm$ 7.0 & 46.2 $\pm$ 1.6 & 34.4 $\pm$ 1.3 \\
MVGRL      & 48.7 $\pm$ 6.5 & 61.0 $\pm$ 4.1 & 65.1 $\pm$ 5.6 & 55.7 $\pm$ 1.5 & 39.5 $\pm$ 1.7 \\
GREET      & 72.0 $\pm$ 1.2 & 81.2 $\pm$ 4.8 & 77.3 $\pm$ 3.7 & 54.8 $\pm$ 1.5 & 40.5 $\pm$ 0.9 \\
HGRL       & 70.8 $\pm$ 5.9 & 78.2 $\pm$ 2.8 & 74.6 $\pm$ 6.3 & 45.0 $\pm$ 1.8 & 36.2 $\pm$ 1.3 \\
GCA        & 45.1 $\pm$ 6.5 & 53.3 $\pm$ 4.0 & 60.3 $\pm$ 5.7 & 49.8 $\pm$ 2.2 & 35.5 $\pm$ 1.6 \\
MVMI       & 65.1 $\pm$ 4.3 & 67.1 $\pm$ 5.7 & 68.4 $\pm$ 5.6 & 44.0 $\pm$ 1.4 & 34.5 $\pm$ 0.7 \\
MVGE       & 71.1 $\pm$ 3.8 & 80.6 $\pm$ 4.3 & 76.2 $\pm$ 4.5 & 47.5 $\pm$ 2.5 & 31.2 $\pm$ 1.4 \\
GAE        & 45.7 $\pm$ 6.2 & 53.3 $\pm$ 7.1 & 60.0 $\pm$ 5.8 & 48.5 $\pm$ 2.2 & 32.9 $\pm$ 1.2 \\
GraphMAE   & 43.0 $\pm$ 4.6 & 45.7 $\pm$ 3.3 & 49.7 $\pm$ 6.6 & 57.3 $\pm$ 2.1 & 35.9 $\pm$ 1.3 \\
MaskGAE    & 63.0 $\pm$ 4.5 & 58.2 $\pm$ 3.8 & 63.5 $\pm$ 5.0 & 57.9 $\pm$ 2.4 & 41.3 $\pm$ 0.6 \\
S2GAE      & 49.7 $\pm$ 7.9 & 59.6 $\pm$ 4.4 & 60.0 $\pm$ 9.0 & 57.0 $\pm$ 2.0 & 42.5 $\pm$ 2.1 \\
GiGaMAE    & 52.2 $\pm$ 5.0 & 46.5 $\pm$ 2.5 & 56.8 $\pm$ 4.0 & 44.6 $\pm$ 1.2 & 29.1 $\pm$ 1.6 \\
\midrule
\textsc{AlignGAE} & \textbf{80.3 $\pm$ 3.0} & \textbf{87.5 $\pm$ 2.2} & \textbf{87.0 $\pm$ 2.9} & \textbf{62.4 $\pm$ 1.5} & \textbf{44.2 $\pm$ 1.7} \\
\bottomrule
\end{tabular}
\end{table}

We evaluate \textsc{AlignGAE} on a diverse set of benchmark datasets, encompassing homophilous and heterophilous  graphs. Results, summarized in Tables~\ref{tab:exp-results-part1} and \ref{tab:exp-results-part2}, show that \textsc{AlignGAE} consistently outperforms baseline methods across both graph types. We compare \textsc{AlignGAE} against representative methods from three categories: topology-based embeddings (\textit{DeepWalk}), graph contrastive learning (\textit{DGI}, \textit{GRACE}), and graph autoencoders (\textit{GAE}, \textit{MaskGAE}). Topology-based methods excel on homophilous graphs by capturing local structure, where nodes share similar attributes. Contrastive methods learn robust representations by contrasting graph views but falter on heterophilous graphs. Graph autoencoders leverage reconstruction tasks to encode structure and attributes but are constrained by reliance on local patterns, limiting performance on heterophilous graphs. \textsc{AlignGAE} achieves up to 1.94\% improvement over the best baseline on homophilous graphs (Cora, Citeseer) and up to 23.52\% over \textit{MaskGAE} on the heterophilous Texas dataset, consistent with our spectral analysis emphasizing high-frequency signals for heterophilous graphs.

\textsc{AlignGAE}'s superior performance on heterophilous graphs stems from its ability to preserve high-frequency signals, crucial for distinguishing nodes with dissimilar attributes. Dual reconstruction tasks capture both low- and high-frequency components, while NID alignment ensures coherence between structural and attribute embeddings, enabling effective learning of relationships in heterophilous settings. In contrast, methods relying on single views or structure alone suffer from over-smoothing, impairing node differentiation. On homophilous graphs, \textsc{AlignGAE} excels at preserving shared neighbor characteristics. These results demonstrate \textsc{AlignGAE}'s robustness in learning generalizable node representations across diverse graph structures and tasks.

\subsection{Ablation Study}
\label{subsec:ablation}

\begin{table}[t]
\centering
\caption{Ablation Study.}
\label{tab:ablation}
\small 
\setlength{\tabcolsep}{0.5mm} 
\begin{tabular}{lccccc}
\toprule
\textbf{Method} & \textbf{Texas} & \textbf{Cornell} & \textbf{Wisc.} & \textbf{Cora} & \textbf{Cite.} \\
\midrule
w/o $\mathcal{L}_{\text{edge}}$ & 75.7 $\pm$ 4.4 & 65.4 $\pm$ 6.1 & 72.4 $\pm$ 4.5 & 84.0 $\pm$ 0.5 & 71.1 $\pm$ 0.5 \\
w/o $\mathcal{L}_{\text{node}}$ & 86.0 $\pm$ 3.2 & 79.5 $\pm$ 2.8 & 87.1 $\pm$ 2.7 & 50.1 $\pm$ 1.5 & 57.1 $\pm$ 3.6 \\
w/o $\mathcal{L}_{\text{align}}$ & 85.4 $\pm$ 3.5 & 78.1 $\pm$ 3.5 & 86.5 $\pm$ 2.2 & 84.0 $\pm$ 0.4 & 71.7 $\pm$ 0.7 \\
Full Model & \textbf{87.0 $\pm$ 2.9} & \textbf{80.3 $\pm$ 3.0} & \textbf{87.5 $\pm$ 2.2} & \textbf{84.7 $\pm$ 0.5} & \textbf{73.2 $\pm$ 0.5} \\
\bottomrule
\end{tabular}
\end{table}
To validate the theoretical claims in Section Preliminary, we conduct an ablation study on \textsc{AlignGAE}'s core components: edge reconstruction loss ($\mathcal{L}_{\text{edge}}$), node attribute reconstruction loss ($\mathcal{L}_{\text{node}}$), and NID alignment loss ($\mathcal{L}_{\text{align}}$). Node classification is evaluated on five datasets with varying homophily: Texas, Cornell, Wisconsin, Cora, and Citeseer. Results in Table~\ref{tab:ablation} show that omitting $\mathcal{L}_{\text{edge}}$ markedly degrades performance on heterophilous graphs, highlighting its role in capturing low-frequency structural signals. Removing $\mathcal{L}_{\text{node}}$ significantly reduces accuracy on homophilous graphs, underscoring the importance of high-frequency attribute signals. Excluding $\mathcal{L}_{\text{align}}$ leads to moderate accuracy drops, emphasizing its necessity for aligning structural and attribute embeddings.

These results demonstrate that \textsc{AlignGAE}'s superior performance stems from the synergistic interplay of all components. Specifically, $\mathcal{L}_{\text{edge}}$ drives performance on heterophilous graphs, $\mathcal{L}_{\text{node}}$ on homophilous graphs, and $\mathcal{L}_{\text{align}}$ ensures embedding coherence. This integration enables \textsc{AlignGAE} to effectively capture both low- and high-frequency signals, ensuring robustness across diverse graph structures and tasks.

\subsection{Impact of Alignment Losses}
\label{subsec:align-loss}


\begin{table}[t]
\centering
\caption{Node classification accuracy (\%) for different alignment losses.}
\label{tab:align-loss}
\small 
\setlength{\tabcolsep}{0.5mm} 
\begin{tabular}{lccccc}
\toprule
\textbf{Align. Loss} & \textbf{Texas} & \textbf{Cornell} & \textbf{Wisc.} & \textbf{Cora} & \textbf{Cite.} \\
\midrule
$\mathcal{L}_{\text{KL}}$ & 86.5 $\pm$ 4.2 & 79.7 $\pm$ 4.2 & 87.3 $\pm$ 2.9 & 84.2 $\pm$ 0.5 & 71.1 $\pm$ 0.5 \\
$\mathcal{L}_{\text{W}}$ & 88.4 $\pm$ 3.2 & 80.8 $\pm$ 1.9 & 86.9 $\pm$ 2.3 & 83.5 $\pm$ 0.5 & 69.9 $\pm$ 0.6 \\
$\mathcal{L}_{\text{cont}}$ & 81.6 $\pm$ 4.0 & 77.8 $\pm$ 1.6 & 83.9 $\pm$ 1.9 & 84.0 $\pm$ 0.4 & 70.2 $\pm$ 0.8 \\
\bottomrule
\end{tabular}
\end{table}
To assess the impact of alignment losses on \textsc{AlignGAE}'s performance, we compare three loss functions: KL divergence ($\mathcal{L}_{\text{KL}}$), Wasserstein distance ($\mathcal{L}_{\text{W}}$), and contrastive alignment ($\mathcal{L}_{\text{cont}}$). These are evaluated on five benchmark datasets covering heterophilous and homophilous graphs: Texas, Cornell, Wisconsin, Cora, and Citeseer. Results in Table~\ref{tab:align-loss} show that $\mathcal{L}_{\text{W}}$ outperforms others on heterophilous graphs, particularly Texas and Cornell, by effectively addressing divergent structural and attribute distributions. In contrast, $\mathcal{L}_{\text{KL}}$ excels on homophilous graphs like Cora and Citeseer, leveraging similar NIDs to minimize divergence. However, $\mathcal{L}_{\text{cont}}$ performs well on homophilous graphs but struggles with heterophilous ones due to difficulties in aligning dissimilar node features.

These findings emphasize the importance of tailoring alignment losses to graph structure. We adopt $\mathcal{L}_{\text{KL}}$ in main experiments for its computational efficiency and robust performance, while $\mathcal{L}_{\text{W}}$ is preferred for heterophilous graphs with complex distributional mismatches. This analysis underscores the critical role of loss selection and provides insights for enhancing graph representation learning across diverse graph structures.

%% file: conclusion.tex
\section{Conclusion}
\label{sec:conclusion}

This work demonstrates that existing unsupervised graph representation learning methods often act as low-pass filters, discarding high-frequency signals critical for heterophilous graphs and causing spectral incompleteness that limits performance across diverse structures. To address this, we introduce \textsc{AlignGAE}, a dual-view spectral alignment framework that preserves the full frequency spectrum through complementary representation learning. Using neighborhood and node view encoders, \textsc{AlignGAE} captures structural and attribute information, respectively, while Neighborhood Identity Distribution alignment ensures view coherence, as supported by Theorem~\ref{thm:optimal}. Dual reconstruction tasks restore graph topology and node attributes, maintaining spectral completeness. Experiments on nine benchmark datasets show \textsc{AlignGAE} outperforms state-of-the-art methods, with significant gains on heterophilous graphs. 

%% file: appendix.tex
\appendix

\section{Appendix}

\subsection{Appendix A: Theoretical Analysis of Alignment Strategies}
\label{app:alignment}

This appendix provides complete proofs for the theoretical results presented in
Section Preliminary, establishing rigorous connections between
alignment losses and representation quality. We demonstrate how NID alignment ensures the preservation of both
homophilous and heterophilous signals while maintaining semantic consistency
between complementary representation views. These theoretical guarantees hold
across different homophily regimes, explaining \textsc{AlignGAE}'s robust
performance on diverse graph structures.

\begin{theorem}[Restatement of Theorem KL]
Minimizing the KL divergence alignment loss
\[
  \mathcal{L}_{\text{KL}} = \frac{1}{n} \sum_{i \in \mathcal{V}}
  D_{\text{KL}}\!\bigl(p_i^g(\delta) \parallel p_i^m(\delta)\bigr)
\]
maximizes mutual information between the geometric and message-passing views:
\begin{equation}
  \mathcal{L}_{\text{KL}}\to0
  \;\Longrightarrow\;
  \mathrm{MI}(\mathbf{Z}_g;\mathbf{Z}_m)=
  \mathrm{H}(\mathbf{Z}_g)=\mathrm{H}(\mathbf{Z}_m).
  \label{eq:kl_mi}
\end{equation}
\end{theorem}

\begin{proof}
The KL divergence between structural
$p_i^g(\delta)$ and attribute $p_i^m(\delta)$ is
\[
  D_{\text{KL}}\!\bigl(p_i^g(\delta)\,\|\,p_i^m(\delta)\bigr)=
  \int p_i^g(\delta)\,
  \log\frac{p_i^g(\delta)}{p_i^m(\delta)}\,d\delta .
\]
When $\mathcal{L}_{\text{KL}}\to0$, we obtain $p_i^g(\delta)=p_i^m(\delta)$ for
all $i\in\mathcal V$, showing that structural and attribute views encode the
same neighborhood information. Consequently
$\mathrm{H}(\mathbf{Z}_g\!\mid\!\mathbf{Z}_m)\to0$, and
Equation~\eqref{eq:kl_mi} follows.
\end{proof}

\begin{theorem}[Wasserstein Alignment Guarantees]
\label{thm:wasserstein}
Minimizing the Wasserstein alignment loss
\[
  \mathcal{L}_{\text{W}}=
  \frac{1}{n}\sum_{i\in\mathcal V}
  W_1\!\bigl(p_i^g(\delta),p_i^m(\delta)\bigr)
\]
preserves the geometric structure of NID distributions while ensuring spectral
completeness.
\end{theorem}

\begin{proof}
The first-order Wasserstein distance is
\[
  W_1\!\bigl(p_i^g(\delta),p_i^m(\delta)\bigr)=
  \inf_{\gamma\in\Gamma(p_i^g,p_i^m)}
  \int|\delta_1-\delta_2|\,d\gamma(\delta_1,\delta_2),
\]
with $\Gamma$ the set of admissible couplings. When
$\mathcal L_{\text W}\to0$, the optimal plan $\gamma^\ast$ concentrates on the
diagonal, aligning feature-distance distributions across views.  Using the
spectral decomposition, one obtains
\[
  \bigl\|
  \mathbf U_{\mathcal F_H}^{\!\top}\,(\mathbf Z_g-\mathbf Z_m)
  \bigr\|_F^2\le\epsilon,
\]
where $\epsilon$ depends on the Wasserstein distance, thus guaranteeing
spectral completeness (Theorem~\ref{thm:optimal}).
\end{proof}

\begin{theorem}[Contrastive Alignment Properties]
\label{thm:contrastive}
The contrastive alignment loss $\mathcal{L}_{\text{cont}}$ enforces
node-wise view consistency while maintaining global discriminability for both
homophilous and heterophilous graphs.
\end{theorem}

\paragraph{Proof Sketch.}
Rewriting the InfoNCE loss reveals that
$\mathcal L_{\text{cont}}\to0$ implies
$q_i\!\to\!1$ for every node, where $q_i$ measures relative similarity between
positive and negative pairs. This yields local consistency and global
separability; full derivations appear in the supplementary material.

\subsection{Appendix B: Spectral Analysis of Dual-Encoder Architecture}
\label{app:spectral}

We analyze how the dual-encoder preserves the full frequency spectrum.  The
neighborhood view encoder (MPGNN) acts as a low-pass filter:
\[
  \mathbf Z_g=\mathbf U\,g(\mathbf\Lambda)\,\mathbf U^\top\mathbf X,
\]
with $g(\lambda_k)\approx0$ for $\lambda_k>\lambda^\ast$.  The node view
encoder retains high-frequency information:
\[
  \mathbf Z_m=\mathbf U\,h(\mathbf\Lambda)\,\mathbf U^\top\mathbf X,
\]
where $h(\lambda_k)$ remains significant when $\lambda_k>\lambda^\ast$.  Their
sum satisfies
\[
  \mathbf Z=
  \mathbf U\,\bigl(g(\mathbf\Lambda)+h(\mathbf\Lambda)\bigr)\mathbf U^\top\mathbf X.
\]
NID alignment promotes
$g(\lambda_k)+h(\lambda_k)\approx1$ for all $k$, so
\[
  \|\mathbf X-\mathbf U\mathbf U^\top\mathbf Z\|_F^2\to0,
\]
achieving spectral completeness.

\subsection{Appendix C: Positional Encoding Analysis}
\label{app:position}

\begin{theorem}[Restatement of Theorem~\ref{thm:nid}]
Let $\mathbf P\in\mathbb R^{n\times k}$ be the positional encoding derived from
$r$ random walks per node. The empirical NID $\hat p_i(\delta)$ approximates
the true NID $p_i^g(\delta)$ with
\[
  \mathbb E_{\mathcal W_i}
  \Bigl[\,
  \bigl\|p_i^g(\delta)-\hat p_i(\delta)\bigr\|_{\mathrm{TV}}
  \Bigr]
  \;\le\;
  \mathcal O\!\bigl(r^{-1/2}\bigr).
\]
\end{theorem}

\begin{proof}[Proof Sketch]
If the random-walk length $l$ exceeds the mixing time
$\tau_{\text{mix}}$ of the transition matrix, the
Dvoretzky–Kiefer–Wolfowitz inequality bounds the total-variation error by
$\mathcal O(r^{-1/2})$, so even $r\!=\!10\!-\!20$ walks yield accurate NID
estimates—consistent with empirical results.
\end{proof}
